\documentclass{article}

\usepackage{microtype}
\usepackage{graphicx}
\usepackage{subfigure}
\usepackage{booktabs} 

\usepackage{hyperref}

\usepackage[accepted]{icml2024}

\usepackage{amsmath}
\usepackage{amssymb}
\usepackage{mathtools}
\usepackage{amsthm}

\usepackage[capitalize,noabbrev]{cleveref}

\theoremstyle{plain}
\newtheorem{theorem}{Theorem}[section]

\newtheorem{lemma}[theorem]{Lemma}
\newtheorem{corollary}[theorem]{Corollary}
\theoremstyle{definition}
\newtheorem{definition}[theorem]{Definition}
\newtheorem{assumption}[theorem]{Assumption}
\theoremstyle{remark}

\usepackage[textsize=tiny]{todonotes}

\icmltitlerunning{Regularized Distribution Matching Distillation for One-step Unpaired Image-to-Image Translation}

\begin{document}
\newcommand{\x}{\boldsymbol{x}}
\newcommand{\xt}{\boldsymbol{x}_t}
\newcommand{\y}{\boldsymbol{y}}
\newcommand{\yt}{\boldsymbol{y}_t}
\newcommand{\bu}{\boldsymbol{u}}
\newcommand{\bv}{\boldsymbol{v}}
\newcommand{\unit}{\mathbf{1}}

\newcommand{\f}{\boldsymbol{f}}
\newcommand{\w}{\boldsymbol{w}}
\newcommand{\wt}{\boldsymbol{w}_t}
\newcommand{\Tmax}{T}
\newcommand{\sigmax}{\sigma_{\text{max}}}
\newcommand{\rmd}{\mathrm{d}}

\newcommand{\score}{\nabla \log}
\newcommand{\scorex}{\nabla_{\x} \log}
\newcommand{\scorey}{\nabla_{\y} \log}
\newcommand{\scorext}{\nabla_{\xt} \log}
\newcommand{\scoreyt}{\nabla_{\yt} \log}
\newcommand{\pnoise}{p^{\text{noise}}}
\newcommand{\pfake}{p^{\phi}}
\newcommand{\preal}{p^{\text{real}}}
\newcommand{\pgen}{p^{G}}
\newcommand{\ptheta}{p^{\theta}}
\newcommand{\px}{p^{\x}}
\newcommand{\py}{p^{\y}}
\newcommand{\psource}{p^{\mathcal{S}}}
\newcommand{\ptarget}{p^{\mathcal{T}}}
\newcommand{\s}{\boldsymbol{s}}
\newcommand{\st}{\boldsymbol{s}_t}

\newcommand{\D}{D}
\newcommand{\sfake}{\boldsymbol{s}^{\phi}}
\newcommand{\sreal}{\boldsymbol{s}^{\text{real}}}
\newcommand{\sdata}{s^{\text{data}}}
\newcommand{\stheta}{\s^\theta}
\newcommand{\starget}{\s^{\mathcal{T}}}

\newcommand{\Dreal}{D^{\text{real}}}
\newcommand{\Dfake}{D^{\phi}}

\newcommand{\KL}{\text{KL}}
\newcommand{\klweight}{\omega_t\,}
\newcommand{\smweight}{\beta_t\,}

\newcommand{\E}{\mathbb{E}}
\newcommand{\N}{\mathcal{N}}
\newcommand{\eps}{\varepsilon}
\newcommand{\z}{\boldsymbol{z}}

\newcommand{\map}{G_\theta}
\newcommand{\loss}{\mathcal{L}}

\newcommand{\R}{\mathbb{R}}
\newcommand{\Rd}{\mathbb{R}^d}
\newcommand{\C}{\mathcal{C}}
\newcommand{\Cb}{\mathcal{C}_b}
\newcommand{\weak}{\xrightarrow[]{w}}

\newcommand{\pdata}{p^{\text{data}}}
\newcommand{\tinit}{\tau}

\twocolumn[
\icmltitle{Regularized Distribution Matching Distillation for One-step Unpaired Image-to-Image Translation}



\icmlsetsymbol{equal}{*}

\begin{icmlauthorlist}
\icmlauthor{Denis Rakitin}{equal,HSE}
\icmlauthor{Ivan Shchekotov}{equal,HSE,SK}
\icmlauthor{Dmitry Vetrov}{CUB}
\end{icmlauthorlist}

\icmlaffiliation{HSE}{HSE University, Moscow, Russia}
\icmlaffiliation{CUB}{Constructor University, Bremen, Germany}
\icmlaffiliation{SK}{Skolkovo Institute of Science and Technology, Moscow, Russia}

\icmlcorrespondingauthor{Denis Rakitin}{rakitindenis32@gmail.com}
\icmlcorrespondingauthor{Ivan Shchekotov}{ivanxshchekotov@gmail.com}

\icmlkeywords{Machine Learning, ICML}

\vskip 0.3in
]



\printAffiliationsAndNotice{\icmlEqualContribution} 

\begin{abstract}
Diffusion distillation methods aim to compress the diffusion models into efficient one-step generators while trying to preserve quality. Among them, Distribution Matching Distillation (DMD) offers a suitable framework for training general-form one-step generators, applicable beyond unconditional generation. In this work, we introduce its modification, called Regularized Distribution Matching Distillation, applicable to unpaired image-to-image problems. We demonstrate its empirical performance in application to several translation tasks, including 2D examples and I2I between different image datasets, where it performs on par or better than multi-step diffusion baselines.
\end{abstract}

\section{Introduction}
\label{sec:intro}
One of the global problems of contemporary generative modeling consists of solving the so-called generative learning trilemma~\cite{xiao2021tackling}. It states that a perfect generative model should possess three desirable properties: high generation quality, mode coverage/diversity of samples and efficient inference. Today, most model families tend to have only 2 of the 3. Generative Adversarial Networks (GANs)~\cite{goodfellow2014generative} have fast inference and produce high-quality samples but tend to underrepresent some modes of the data set~\cite{metz2016unrolled, arjovsky2017wasserstein}. Variational Autoencoders (VAEs)~\cite{kingma2013auto, rezende2014stochastic} efficiently produce diverse samples while suffering from insufficient generation quality. Finally, diffusion-based generative models~\cite{ho2020denoising, song2020score, dhariwal2021diffusion,  karras2022elucidating} achieve SOTA generative metrics and visual quality yet require running a high-cost multi-step inference procedure.

Satisfying these three properties is essential in numerous generative computer vision tasks beyond unconditional generation. One is image-to-image (I2I) translation~\cite{isola2017image, zhu2017unpaired}, which consists of learning a mapping between two distributions that preserves the cross-domain properties of an input object while appropriately changing its source-domain features to match the target. Most examples, like transforming cats into dogs~\cite{choi2020stargan} or human faces into anime~\cite{korotin2022neural} belong to the \emph{unpaired} I2I because they do not assume ground truth pairs of objects in the data set. As in unconditional generation, unpaired I2I methods were previously centered around GANs~\cite{huang2018multimodal, park2020contrastive, choi2020stargan, zheng2022ittr}, but now tend to be competed and surpassed by diffusion-based counterparts~\cite{choi2021ilvr, meng2021sdedit, zhao2022egsde, wu2023latent}. Most of these methods build on top of the original diffusion sampling procedure and tend to have high generation time as a consequence.

Since diffusion models succeed in both desirable qualitative properties of the trilemma, one could theoretically obtain samples of the desired quality level given sufficient computational resources. It makes the acceleration of diffusion models an appealing approach to satisfy all of the aforementioned requirements, including efficient inference. 

Recently introduced diffusion distillation techniques~\cite{song2023consistency, kim2023consistency, sauer2023adversarial} address this challenge by compressing diffusion models into one-step students with (hopefully) similar qualitative and quantitative properties. Among them, Distribution Matching Distillation (DMD)~\cite{yin2023one, nguyen2023swiftbrush} offers an expressive and general framework for training free-form generators based on techniques initially introduced for text-to-3D~\cite{poole2022dreamfusion, wang2024prolificdreamer}. \emph{Free-form} here means that the method does not make any assumptions about the generator's structure and distribution at the input. This crucial observation opens a large space for its applications beyond the $noise \rightarrow data$ problems.

\begin{figure*}[!t]
\begin{center}
\includegraphics[width=\textwidth]{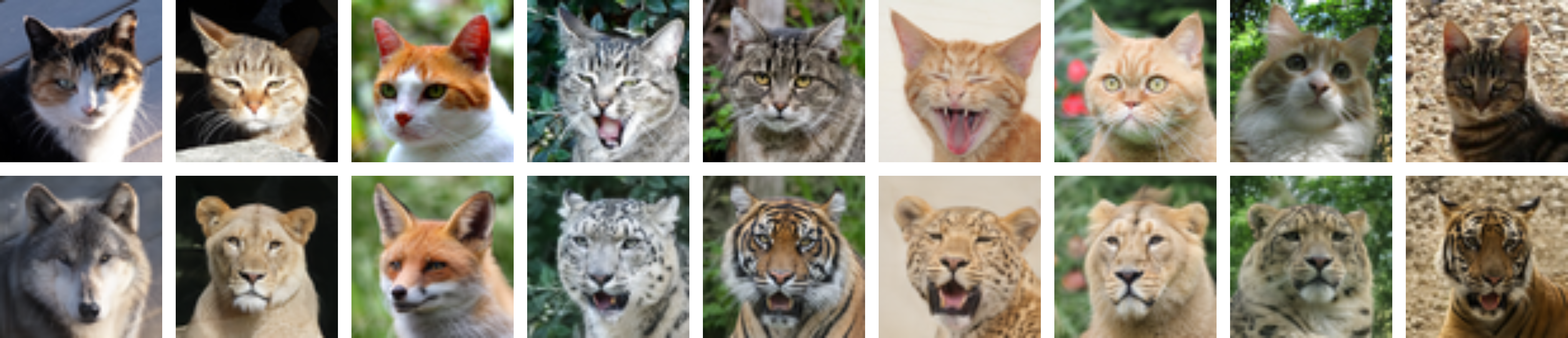}
\caption{Illustration of performance of the proposed RDMD model on $cat\rightarrow wild$ translation problem/  from the AFHQv2~\cite{choi2020stargan} data set.
}
\label{fig:title_pic}
\end{center}
\end{figure*}

In this work, we introduce the modification of DMD, called \emph{Regularized Distribution Matching Distillation} (RDMD), that applies to the unpaired I2I problems. To achieve this, we replace the generator's input noise with the source data samples to further translate them into the target. We maintain correspondence between the generator's input and output by regularizing the objective with the transport cost between them. As our main contributions, we
\begin{enumerate}
    \item Propose a one-step diffusion-based method for unpaired I2I; 
    \item Theoretically verify it by establishing its connection with optimal transport~\cite{villani2009optimal, peyre2019computational};
    \item Ablate its qualitative properties and demonstrate its generation quality on 2D and image-to-image examples, where it obtains comparable or better results than the multi-step counterparts.
\end{enumerate}

\section{Background}
\label{sec:background}

\subsection{Diffusion Models}
Diffusion models \cite{song2019generative, ho2020denoising} are a class of models that sequentially perturb data distribution $\pdata$ with Gaussian noise, transforming it into some tractable unstructured distribution, which contains no information about initial domain. 

Using this distribution as a prior and reversing the process by progressively removing the noise yields a sampling procedure from $\pdata$. A convenient way to formalize diffusion models is through stochastic differential equations (SDEs) \cite{song2020score}, which describe continuous-time stochastic dynamics of particles.
The forward process is commonly defined as the Variance Exploding (VE) SDE\footnote{The other popular forward processes (e.g. VP-SDE) can be obtained by scaling the VE-SDE.}
\begin{equation}
\label{eq:sde_fwd}
    \rmd \xt = g(t) \rmd \wt,
\end{equation}
where $t \in [0, \Tmax]$, $\x_0 \sim \pdata$, $g(\cdot)$ is the scalar diffusion coefficient and $\rmd \wt$ is the differential of a standard Wiener process.
We denote by $p_t(\xt)$ marginal distribution of $\xt$, so that $\pdata(\x_0) = p_0(\x_0)$. $p_{\Tmax}$ acts as an unstructured prior distribution that we can sample from.

Conveniently, SDE dynamics can be represented via a deterministic counterpart given by an ordinary differential equation (ODE), which yields the same marginal distributions $p_t(\xt)$ as in Equation~\ref{eq:sde_fwd}, given the same initial distribution $p_0(\x_0) = \pdata(\x_0)$:
\begin{equation}
\label{eq:pf_ode}
\rmd \xt = -\frac{1}{2}g^2(t) \scorex p_t(\xt) \rmd t,
\end{equation}
where $\scorext p_t(\xt)$ is called the \emph{score function} of $p_t(\xt)$. Equation~\ref{eq:pf_ode} is also called Probability Flow ODE (PF-ODE). The ODE formulation allows us to obtain a backward process by simply reversing velocity of the particle. In particular, we can obtain samples from $\pdata$ by taking $\x_{\Tmax} \sim p_{\Tmax}$ and running the PF-ODE backwards in time, given access to the score function.

However, in the case of generative modeling $\scorex p_t(\xt)$ is intractable due to $\pdata$ being intractable, and thus cannot be used directly in Equation~\ref{eq:pf_ode}. Under mild regularity conditions, the unconditional score can be expressed by:
\begin{equation}
\label{eq:uncond_as_condexp}
\scorext p_t (\x_t) = \E_{p_{0 | t}(\x_0 | \x_t)} \left[\s_{t | 0}(\x_t | \x_0)\right],
\end{equation}
where $\s_{t | 0}(\x_t | \x_0) = \scorext p_{t | 0}(\x_t | \x_0)$ is the conditional distribution (also called perturbation kernel) and $p_{0 | t}(\x_0 | \x_t)$ is the corresponding posterior distribution. 
The perturbation kernel in the case of VE-SDE corresponds to simply adding an independent Gaussian noise: 
\begin{equation}
p_{t | 0}(\x_t | \x_0) = \N(\x_t | \x_0, \, \sigma_t^2 I), \: \sigma_t^2 = \int_{0}^{t} g^2(s) \rmd s.
\end{equation}
Denoising Score Matching (DSM) \cite{Vincent2011ACB} utilizes Equation~\ref{eq:uncond_as_condexp} and approximates $\scorext p_t (\x_t)$ with the score model $\st^\theta(\x_t)$ via L2-regression minimization:
\begin{equation}
\label{eq:dsm_loss}
\int\limits_{0}^{\Tmax}\smweight \E_{p_{0, t}(\x_0, \x_t)} \| \st^\theta(\x_t) - \s_{t | 0}(\x_t | \x_0) \|^2 \rmd t \rightarrow \min\limits_{\theta},
\end{equation}
where $\smweight$ is some positive weighting function. The minimum in the Equation~\ref{eq:dsm_loss} is obtained at $\st^\theta(\x_t) = \scorext p_t(\x_t)$. Given a suitable parameterization of the score network, DSM objective is equivalent to
\begin{equation}
\label{eq:dsm_loss_mean}
\int\limits_{0}^{\Tmax} \smweight \E_{p_{0, t}(\x_0, \x_t)}\| \D^{\theta}_{t}(\x_t) - \x_0 \|^2 \rmd t \rightarrow \min\limits_{\theta},
\end{equation}
where $\D^{\theta}_{t}$ is called the denoising network (or simply denoiser) and is related to the score network via $\st^\theta(\x_t) = \left(\x_t - \D_t^\theta(\x_t)\right)/ \sigma_t^2$. Therefore, Denoising Score Matching procedure consists of learning to denoise images at various noise levels.

Having obtained $\st^\theta(\x_t)$, we solve Equation~\ref{eq:pf_ode} backward in time, starting from $x_{\Tmax} \sim \N(0, \sigma_{\Tmax}^2 I)$ to obtain approximate samples from $\pdata$.


\subsection{Distribution Matching Distillation}
Distribution Matching Distillation~\cite{yin2023one} is the core technique of this paper. Essentially, it aims to train a generator $\map(\z)$ on matching the given distribution $\preal$. Its input $\z$ is assumed to come from a tractable input distribution $\pnoise$. Formally, matching two distributions can be achieved by optimizing the KL divergence between the distribution\footnote{The superscript $\theta$ in $\ptheta$ does not mean introducing the additional neural model of density but is rather used to emphasize its dependence on the generator.} $\ptheta$ of $\map(\z)$ and the data distribution $\preal$:
\begin{equation}
\label{eq:dmd_kl}
    \KL(\ptheta \,\|\, \preal) = \E_{\pnoise(\z)} \log \frac{\ptheta(\map(\z))}{\preal(\map(\z))} \rightarrow \min\limits_{\theta}
\end{equation}
Differentiating it by the parameters $\theta$, using the chain rule, one encounters a summand, containing the difference $\s^\theta(\map(\z)) - \sreal(\map(\z))$ between the score functions of the corresponding distributions~\footnote{Note that there is one more summand, which contains the gradient $\nabla_\theta \log \ptheta$ with respect to the log-density parameters. We do not discuss how to approximate it, because it will be further omitted.}. The pure data score function can be very non-smooth due to the Manifold Hypothesis~\cite{tenenbaum2000global} and is generally hard to train~\cite{song2019generative}, so the authors make the problem accessible through the diffusion framework. To this end, they replace the original loss with an ensemble of KL divergences between distributions, perturbed by the forward diffusion process:
\begin{equation}
\label{eq:dmd_kl_t} \int_{0}^{\Tmax} \klweight \KL\Big(\ptheta_t \,\|\, \preal_t\Big) \rmd t,
\end{equation}
Here, $\klweight$ is a weighting function, $\ptheta_t$ and $\preal_t$ are the perturbed versions of the generator distribution and $\preal$ up to the time step $t$. In theory, the minima of Equation~\ref{eq:dmd_kl_t} objective coincides~\citep[Thm.~1]{wang2024prolificdreamer} with the original minima from Equation~\ref{eq:dmd_kl}. Meanwhile in practice taking the gradient of the new loss, which can be equivalently written as
\begin{equation}
\label{eq:dmd_kl_t_2} 
    \int\limits_{0}^{\Tmax} \klweight \E_{\mathcal{N}(\eps | 0, I)\pnoise(\z)} \log \frac{\ptheta_t(\map(\z) + \sigma_t \eps)}{\preal_t(\map(\z) + \sigma_t\eps)} \,\rmd t,
\end{equation}
results in obtaining difference $\stheta_t(\map(\z) + \sigma_t \eps) - \sreal_t(\map(\z) + \sigma_t \eps)$, which can be approximated by the diffusion models.

Given this, authors approximate $\sreal_t$ with the pre-trained diffusion model, which we will denote $\sreal_t$ as well with a slight abuse of notation. The whole procedure now can be considered as distillation of $\sreal_t$ into $\map$. At the same time, $\stheta_t$ is the score of the noised distribution of the generator, which is intractable and therefore approximated by an additional "fake" diffusion model $\sfake_t$ and the corresponding denoiser $\Dfake_t$. It is trained on the standard denoising score matching objective with the generator's samples at the input. The joint training procedure is essentially the coordinate descent
\begin{equation}
\label{eq:dmd_training}
\begin{cases}
    \int\limits_{0}^{\Tmax} \klweight \E_{\eps, \z} \log \cfrac{\pfake_t(\map(\z) + \sigma_t \eps)}{\preal_t(\map(\z) + \sigma_t\eps)} \,\rmd t \rightarrow \min\limits_{\theta}; \\
    \int\limits_{0}^{\Tmax} \smweight \E_{\eps, \z} \| \Dfake_t(\map(\z) + \sigma_t \eps) - \map(\z) \|^2 \,\rmd t \rightarrow \min\limits_{\phi},
\end{cases}
\end{equation}
where the stochastic gradient with respect to the fake network is calculated by backpropagation and the generator's stochastic gradient is calculated directly as
\begin{equation}
\label{eq:dmd_gradient}
\klweight \left(\sfake_t - \sreal_t\right) \nabla_\theta \map(\z),
\end{equation}
where the scores are evaluated in the point $\map(\z) + \sigma_t \eps$. Minimization of the fake network's objective ensures $\sfake_t = \stheta_t \Leftrightarrow \pfake_t = \ptheta_t$. At this condition, the generator's objective is equal to the original ensemble of KL divergences from Equation~\ref{eq:dmd_kl_t}, minimizing which solves the initial problem and implies $\ptheta = \preal$.
\subsection{Unpaired I2I and optimal transport}
The problem of unpaired I2I consists of learning a mapping $G$ between the \emph{source} distribution $\psource$ and the \emph{target} distribution $\ptarget$ given the corresponding independent data sets of samples. When optimized, the mapping should appropriately adapt $G(\x)$ to the target distribution $\ptarget$, while preserving the input's cross-domain features. However, from the first glance it is unclear what the preservation of cross-domain properties should be like.

One way to look at that formally is by introducing the notion of "transportation cost" $c(\cdot, \cdot)$ between the generator's input and output and saying that it should not be too large on average. In a practical I2I setting, we can choose $c(\cdot, \cdot)$ as any reasonable distance between images or their features that we aim to preserve, e.g. pixel-wise distance or difference between LPIPS \cite{zhang2018perceptual} embeddings.

Monge optimal transport (OT) problem \cite{villani2009optimal, santambrogio2015optimal} follows this reasoning and aims at finding the mapping with the least average transport cost among all the mappings that fit the target $\ptarget$:
\begin{equation}
\label{eq:monge_ot_problem}
\inf_{G} \left\{ \E_{\psource(\x)} c(\x, G(\x)) \: \lvert \: G(\x) \sim \ptarget \right\},
\end{equation}
which can be seen as a mathematical formalization of the I2I task. 

Under mild constraints, in the case when $\psource$ and $\ptarget$ have densities, the optimal transport map $G^*$ is bijective, differentiable, has differentiable inverse and thus satisfies the change of variables formula $\psource(\x) = \ptarget(G^*(\x))| \det\left(\nabla G^*(\x)\right)|$. This highly non-linear change of variables condition gives insight into why it is notoriously challenging to optimize Equation~\ref{eq:monge_ot_problem} directly.


\section{Methodology}
\label{sec:method}
Our main goal is to adapt the DMD method for the unpaired I2I between an arbitrary source distribution $\psource$ and target distribution $\ptarget$.
\subsection{Regularized Distribution Matching Distillation}
\label{subsec:rdmd}
First, we note that the construction of DMD requires only having samples from the input distribution. Given this, we replace the Gaussian input $\pnoise$ by $\psource$, the data distribution $\pdata$ by $\ptarget$ and aim at optimizing
\begin{multline}
\label{eq:rdmd_kl_t_noreg}
    \loss(\theta) = \int\limits_0^{\Tmax} \klweight \KL \Big( p_t^\theta \,\|\, \ptarget_t \Big) \rmd t = \\
    =\int\limits_0^{\Tmax} \klweight \E_{\psource(\x) \mathcal{N}(\eps | 0, I)} \log \cfrac{\ptheta_t(\map(\x) + \sigma_t \eps)}{\ptarget_t(\map(\x) + \sigma_t \eps)} \,\rmd t,
\end{multline}
where $p^\theta_t$ and $\ptarget_t$ are now respectively the distribution of the generator output $\map(\x)$ and the target distribution $\ptarget$, perturbed by the forward process up to the timestep $t$.

Optimizing the objective in Equation~\ref{eq:rdmd_kl_t_noreg}, one obtains a generator, which takes $\x \sim\psource$ and outputs $\map(\x) \sim \ptarget$, so it performs the desired transfer between the two distributions. However, there are no guarantees that the input and the output will be related. Similarly to the OT problem (Equation~\ref{eq:monge_ot_problem}), we fix the issue by penalizing the transport cost between them. We obtain the following objective
\begin{equation}
\label{eq:rdmd_loss}
\loss^\lambda(\theta) = \loss(\theta) + \lambda\,\E_{\psource(\x)} c\left(\x, \map(\x)\right) \rightarrow\min\limits_{\theta},
\end{equation}
where $c(\cdot, \cdot)$ is the cost function, which describes the object properties that we aim to preserve after transfer,  and $\lambda$ is the regularization coefficient. Choosing the appropriate $\lambda$ will result in finding a balance between fitting the target distribution and preserving properties of the input.

As in DMD, we assume that the perturbed target distributions are  represented by a pre-trained diffusion model $\starget_t$ and approximate the generator distribution score $\stheta_t$ by the additional fake diffusion model $\sfake_t$. Analogous to the DMD procedure (Equation~\ref{eq:dmd_training}), we perform the coordinate descent in which, however, the generator objective is now regularized. We call the procedure \emph{Regularized Distribution Matching Distillation} (RDMD). Formally, we optimize
\begin{align}
\label{eq:rdmd_training}
\begin{cases}
    \int\limits_{0}^{\Tmax} \klweight \E_{\eps, \x} \log \cfrac{\pfake_t(\map(\x) + \sigma_t \eps)}{\ptarget_t(\map(\x) + \sigma_t \eps)} \, \rmd t \\
    \:\:\:\:\:\:\:\:\:\:\:\:\:\:\:\:\:\:\:\:\:\:\:\:\:\:\:\:\:\:\:\:\:\:\:\:+\:\lambda \, \E_{\psource(\x)} c\left(\x, \map(\x)\right) \rightarrow \min\limits_{\theta}; \\
    \int\limits_{0}^{\Tmax} \smweight \E_{\eps, \x} \| \Dfake_t(\map(\x) + \sigma_t \eps) - \map(\x) \|^2 \,\rmd t \rightarrow \min\limits_{\phi}.
\end{cases}
\end{align}
Given the optimal fake score $\sfake_t$, the generator's objective becomes equal to the desired loss in Equation~\ref{eq:rdmd_loss}, which validates the procedure.
\subsection{Analysis of the method}
\label{subseq:analysis}

\begin{figure*}[!t]
\begin{center}
\includegraphics[width=0.9\textwidth]{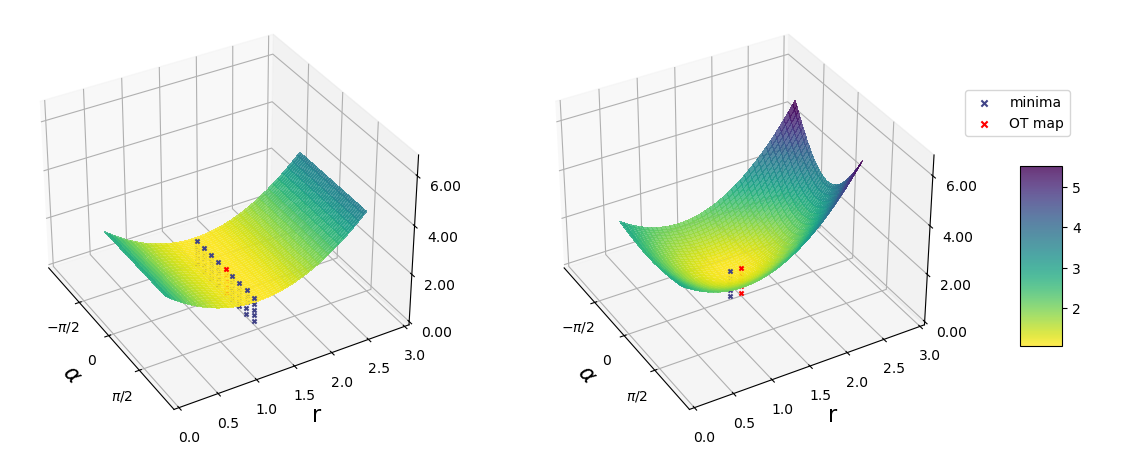}
\caption{Comparison of the DMD loss surfaces without (left) and with (right) transport cost regularization on a toy problem of translating $\mathcal{N}(0, I)$ to $\mathcal{N}(0, 1.5^2 I)$. We set the regularization coefficient $\lambda = 0.2$. The generator is parameterized as $r \cdot C(\alpha)$, where $C(\alpha)$ is the rotation matrix, corresponding to the angle $\alpha$. Minima at the left contains all orthogonal matrices, multiplied by $\sigma = 1.5$, while the minimum at the right is attained in the only point, which is close, but not equal, to the OT map. The surfaces are moved up for the sake of visualization.}
\label{fig:surface}
\end{center}
\end{figure*}

The optimization problem in Equation~\ref{eq:rdmd_loss} can be seen as the soft-constrained optimal transport, which balances between satisfying the output distribution constraint and preserving the original image properties. Moreover, if one takes $\lambda \rightarrow 0$, the objective essentially becomes equivalent to the Monge problem (Equation~\ref{eq:monge_ot_problem}). It can be seen by replacing the $\lambda$ coefficient before the transport cost with the $1/\lambda$ coefficient before the KL divergence. In the limit, it equals $+\infty$ whenever the generator output and the target distributions are different, which makes the corresponding problem hard-constrained and, therefore, equivalent to the original optimal transport problem. Based on this observation, we prove the following
\begin{theorem}
\label{thm1}
    Let $c(\x, \y)$ be the quadratic cost $\|\x - \y\|^2$ and $G^\lambda$ be the theoretical optimum in the problem~\ref{eq:rdmd_loss}. Then, under mild regularity conditions, it converges in probability (with respect to $\psource$) to the optimal transport map $G^*$, i.e.
    \begin{equation}
        G^\lambda \xrightarrow[\lambda \rightarrow 0]{\psource} G^*.
    \end{equation}
\end{theorem}
The detailed proof can be found in Appendix~\ref{sec:theory}. Informally, it means that the optimal transport map can be approximated by the RDMD generator, trained on Equation~\ref{eq:rdmd_training}, given a sufficiently small regularization coefficient, enough capacity of the architecture, and convergence of the optimization algorithm.

This result is important to examine from another angle. It is ideologically similar to the $L_2$ regularization for over-parameterized least squares regression. The original least squares, in this case, have a manifold of solutions. At the same time, by adding $L_2$ weight penalty and taking the limit as the regularization coefficient goes to zero, one obtains a solution with the least norm based on the Moore-Penrose pseudo-inverse~\cite{moore1920reciprocal, penrose1955generalized}. In our case, numerous maps may be optimal in the original DMD procedure, since it only requires matching the distribution at output. However, taking the limit when $\lambda \rightarrow 0$, one obtains a feasible solution with the least transport cost.

We demonstrate this effect on a toy problem of translating $\N(0, I)$ to $\N(0, \sigma^2 I)$ and consider linear generator $G(\x) = A \x$. The solution to the optimal transport problem with the quadratic cost $c(\x, \y) = \| \x - \y \|^2$ is $A = \sigma I$. For the DMD optimization problem without regularization, minima are obtained at the manifold of orthogonal matrices multiplied by $\sigma$: $AA^\top = \sigma^2I$. However, if one adds the regularization, the minimum compresses into one point at the cost of introducing the bias relatively to the true OT map. We illustrate this by comparing the loss surface with and without regularization in Figure~\ref{fig:surface}.

\begin{figure*}[!t]
\centering
\includegraphics[width=0.9\linewidth]{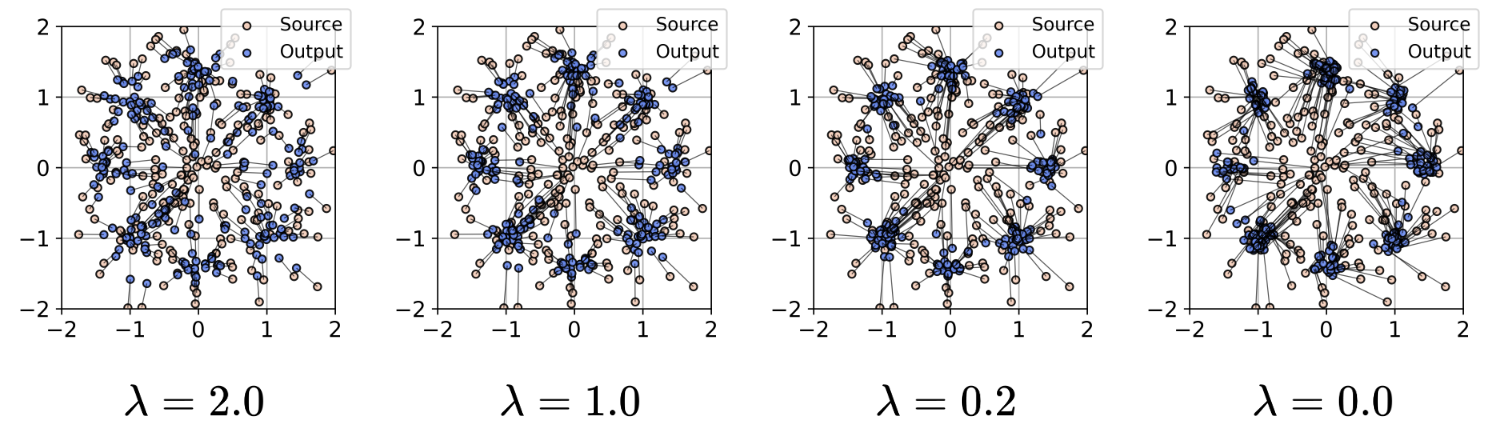}
\caption{Visualization of RDMD mappings on $Gaussian \rightarrow 8 Gaussians$ with different choices of the regularization coefficient $\lambda$.}
\label{fig:8gaussian_mappings}
\end{figure*}

\section{Related work}
\label{sec:related}
In this section, we give an overview of the existing methods for solving unpaired I2I including GANs, diffusion-based methods, and methods based on optimal transport.

\textbf{GANs} were the prevalent paradigm in the unpaired I2I for a long time. Among other methods, CycleGAN~\cite{zhu2017unpaired} and the concurrent DualGAN~\cite{yi2017dualgan}, DiscoGAN~\cite{kim2017learning} utilized the cycle-consistency paradigm, consisting in training the transfer network along with its inverse and optimizing the consistency term along with the adversarial loss. It gave rise to the whole family of two-sided methods, including UNIT~\cite{liu2017unsupervised} and MUNIT~\cite{huang2018multimodal} that divide the encoding into style-space and content-space and SCAN~\cite{li2018unsupervised} that splits the procedure into coarse and fine stages. 

The one-side GAN-based methods aim to train I2I without learning the inverse for better computational efficiency. DistanceGAN~\cite{benaim2017one} achieves it by learning to preserve the distance between pairs of samples, GCGAN~\cite{fu2019geometry} imposes geometrical consistency constraints, and CUT~\cite{park2020contrastive} uses the contrastive loss to maximize the patch-wise mutual information between input and output.

\textbf{Diffusion-based} I2I models mostly build on modifying the diffusion process using the source image. SDEdit~\cite{meng2021sdedit} initializes the reverse diffusion process for target distribution with the noisy source picture instead of the pure noise to maintain similarity. Many methods guide~\cite{ho2022classifier, epstein2023diffusion} the target diffusion process. ILVR~\cite{choi2021ilvr} adds the correction that enforces the current noisy sample to resemble the source. EGSDE~\cite{zhao2022egsde} trains a classifier between domains and encourages dissimilarity between the embeddings, corresponding to the source image and the current diffusion process state. At the same time, it enforces a small distance between their downsampled versions, which allows for a balance between faithfulness and realism. The other diffusion-based approaches include two-sided methods based on the concatenation of two diffusion models with deterministic sampling~\cite{su2022dual, wu2023latent}. 

\textbf{Optimal transport}~\cite{villani2009optimal, peyre2019computational} is another useful framework for the unpaired I2I. Methods based on it usually reformulate the OT problem (Equation~\ref{eq:monge_ot_problem}) and its modifications as Entropic OT (EOT) ~\cite{cuturi2013sinkhorn} or Schrödinger Bridge (SB) ~\cite{follmer1988random} to be accessible in practice. In particular, NOT~\cite{korotin2022neural}, ENOT~\cite{gushchin2024entropic}, and NSB~\cite{kim2023unpaired} use the Lagrangian multipliers formulation of the distribution matching constraint, which results in simulation-based adversarial training. The other methods obtain (partially) simulation-free techniques by iteratively refining the stochastic process between two distributions. In the works~\cite{de2021diffusion, vargas2021solving} refinement consists of learning the time-reversal with the corresponding initial distribution (source or target). The newer methods are based on Flow Matching~\cite{lipman2022flow, tong2023improving, albergo2022building} and the corresponding Rectification~\cite{liu2022flow, shi2024diffusion, liu2023instaflow} procedure. While being theoretically sound, most of these methods work well for smaller dimensions~\cite{korotin2023light} but suffer from computationally hard training in large-scale scenarios.

\section{Experiments}
\label{sec:exp}
This section presents the experimental results on 2 unpaired translation tasks. Section~\ref{subsec:toy_exps} is devoted to the toy 2D experiment. In Section~\ref{subsec:c2w} we compare our method with the diffusion-based baselines on the translation problem between cats and wild animals from the AFHQv2 data set~\cite{choi2020stargan}.

In all the experiments, we use the forward diffusion process with variance $\sigma_t = t$ and $\Tmax=80.0$ as in the paper~\cite{karras2022elucidating}. We parameterize all the diffusion models with the denoiser networks $D_\sigma(\x)$, conditioned on the noise level $\sigma$, and optimize Equation~\ref{eq:dsm_loss_mean} to train the target diffusion model. As for the RDMD procedure, we optimize Equation\ref{eq:rdmd_training}, where the gradient with respect to the generator parameters is calculated analogously to Equation~\ref{eq:dmd_gradient}. The transport cost $c(\x, \y)$ is chosen as the squared difference norm $\|\x - \y\|^2$. The average transport cost, reported in the figures, is calculated as the square root of the MSE between all input and output images for the sake of interpretability.

We use the same architecture for all networks: target score, fake score, and generator. We utilize the pre-trained target score in two ways. First, we initialize the fake model with its copy. Second, we initialize the generator $G_\theta(\x)$ with the same copy $\Dreal_{\sigma}(\x)$, but with a fixed $\sigma \in [0, T]$ (since the generator is one-step). The denoiser parameterization is trained to predict the target domain's clean images, therefore, such initialization should significantly speed up convergence and nudge the model to utilize the information about the target domain more efficiently~\cite{nguyen2023swiftbrush, yin2023one}. We explore the initialization of $\sigma$ for I2I in Appendix~\ref{sec:sigma_ablation}. The additional training details can be found in Appendix~\ref{sec:exp_details}.
\subsection{Toy Experiment}
\label{subsec:toy_exps}
We validate the qualitative properties of the RDMD method on 2-dimensional \textit{Gaussian $\rightarrow$ 8Gaussians}. In this setting, we explore the effect of varying the regularization coefficient $\lambda$ on the trained transport map $\map$. In particular, we study its impact on the transport cost and fitness to the target distribution $\ptarget$. 

In the experiment, both source $\N(0, I)$ and the target mixture of 8 Gaussians are represented with 5000 independent samples. We use the same small MLP-based architecture~\cite{shi2024diffusion} for all the networks.

The main results are presented in Figure \ref{fig:8gaussian_mappings}. The standard DMD ($\lambda=0.0$) learns a transport map with several intersections when demonstrated as the set of lines between the inputs and the outputs. This observation means that the learned map is not OT, because it is not cycle-monotone~\cite{mccann1995existence}. Increasing $\lambda$ yields fewer intersections, which can be used as a proxy evidence of optimality. At the same time, the generator output distribution becomes farther and farther from the desired target. The results show the importance of choosing the appropriate $\lambda$ to obtain a better trade-off between the two properties. Here, the regularization coefficient $\lambda = 0.2$ offers a good trade-off by having small intersections and producing output distribution close to the target.


\begin{figure}[!t]
\centering
\includegraphics[width=0.8\columnwidth]{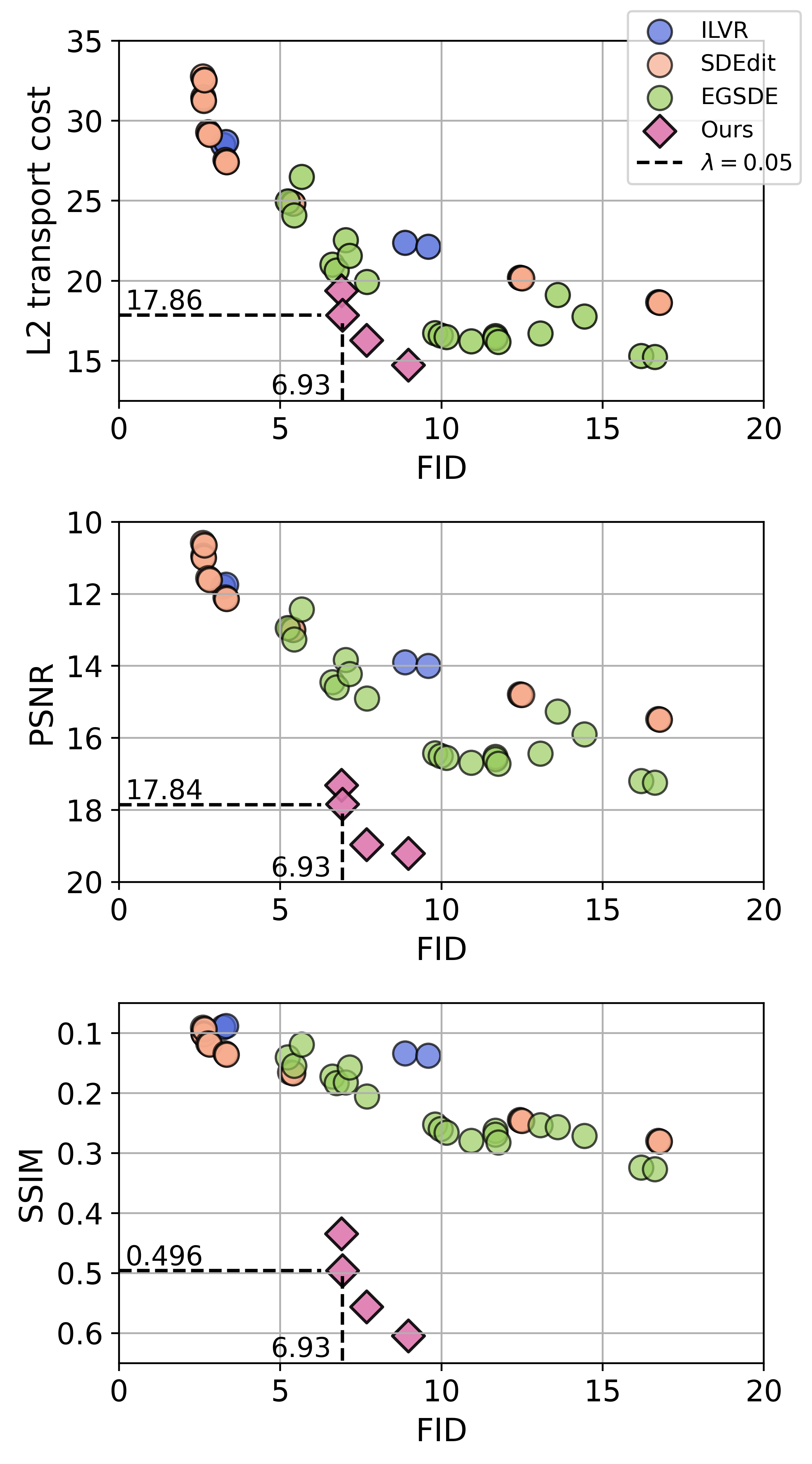}
\caption{Comparison of RDMD with diffusion-based baselines. The figure demonstrates the tradeoff between generation quality (FID$\downarrow$) and the difference between the input and output (L2$\downarrow$, PSNR$\uparrow$, SSIM$\uparrow$). RDMD gives an overall better tradeoff given fairly strict requirements on the transport cost. In the cases of PSNR and SSIM, the $y$-axis is swapped for the sake of identical readability with the first plot (left is better, low is better).
}
\label{fig:c2w_metrics}
\end{figure}
\subsection{Cat to Wild}
\label{subsec:c2w}
Finally, we compare the proposed RDMD method with the diffusion-based baselines ILVR~\cite{choi2021ilvr}, SDEdit~\cite{meng2021sdedit}, and EGSDE~\cite{zhao2022egsde} on the $64 \times 64$ $Cat \rightarrow Wild$ translation problem, based on the AFHQv2~\cite{choi2020stargan} data set. Comparison with the diffusion-based models makes the setting fair since it allows to utilize the same pre-trained target diffusion model for all of the methods. We stress, however, that the GAN-based methods mostly demonstrate results inferior to EGSDE~\cite{zhao2022egsde} in terms of FID and PSNR at the same data set with resolution $256 \times 256$.

\begin{figure}[!t]
\centering
\includegraphics[width=\columnwidth]{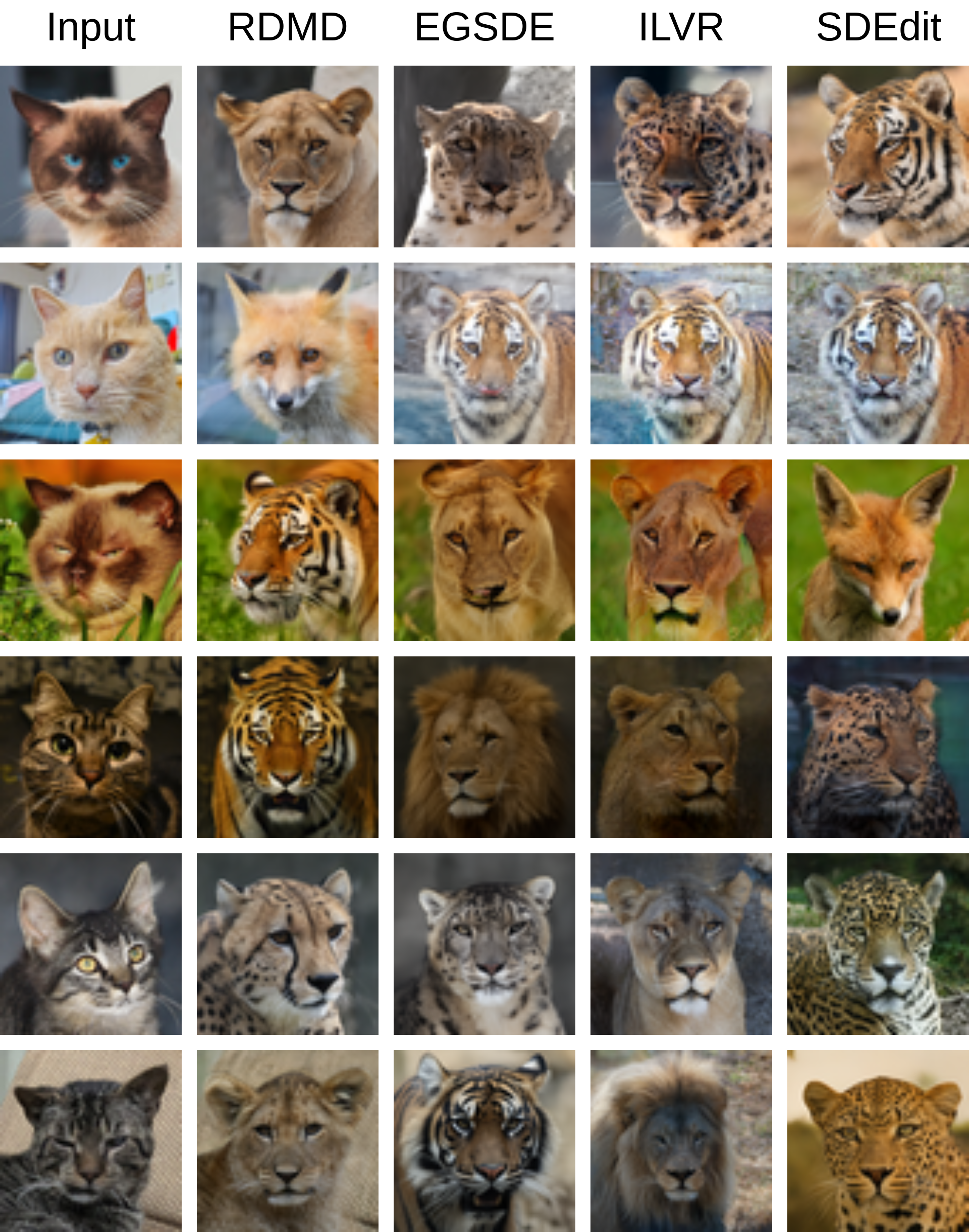}
\caption{Visual comparison of RDMD with diffusion-based baselines.
}
\label{fig:c2w_comparison}
\end{figure}

We pre-train the target diffusion model using the EDM~\cite{karras2022elucidating} architecture with the hyperparameters used by the authors on the AFHQv2 data set. Our pre-trained model achieves FID equal to 2.0. We initialize the model with $\sigma=1.0$ based on the observations from Section~\ref{sec:sigma_ablation} and train 5 RDMD generators, corresponding to the regularization coefficients $\{0.001, 0.02, 0.05, 0.1, 0.2\}$. We slightly adapt the official baseline implementations for compatibility with the EDM setting. The EGSDE classifier is trained analogous to the paper: it is initialized from the Dhariwal UNet~\cite{dhariwal2021diffusion}, pre-trained on the ImageNet~\cite{deng2009imagenet} $64 \times 64$. For each of the baselines, we run a grid of hyperparameters. The detailed hyperparameter values can be found in Appendix~\ref{subsec:c2w_details}.

We report our main quantitative results in Figure~\ref{fig:c2w_comparison} and compare the achieved faithfulness-quality trade-off with the baselines. The quality metric is FID, the faithfulness metrics are L2/PSNR/SSIM. Among these metrics, L2 is the least convenient for our method. Nevertheless, RDMD achieves a better trade-off given at least moderately strict requirements on the transport cost: all of our models beat all the baselines in the L2 range between $12.5$ and $20.0$. In all cases, our model achieves strictly higher SSIM and almost strictly higher PSNR. We note, however, that if the lower FID is preferable over the transport cost (L2 values around $22.5-27.5$), then it might be better to use one of the baselines. An example of a map with a high OT cost (25.0) and low FID (5.4) is SDEdit on Figure~\ref{fig:c2w_comparison}.

Finally, we present a visual comparison between the methods. To this end, we randomly choose 6 pictures from the test data set and report the corresponding outputs in Figure~\ref{fig:c2w_comparison}. Here, we take RDMD with $\lambda=0.05$ that achieves (FID, L2) equal to $(6.93, 17.86)$. As for the baselines, we choose the hyperparameters (Appendix~\ref{subsec:c2w_details}) with the closest FID to the RDMD: $(8.87, 22.0)$ for ILVR, $(5.4, 25.0)$ for SDEdit, and $(7.02, 22.35)$ for EGSDE.

\section{Discussion and limitations}
\label{sec:discuss}
In this paper, we propose RDMD, the novel \emph{one-step} diffusion-based algorithm for the unpaired I2I task. This algorithm is a modification of the DMD method for diffusion distillation. The main novelty is the introduction of the transport cost regularization between the input and the output of the model, which allows to control the trade-off between faithfulness and visual quality.

From the theoretical standpoint, we prove that at low regularization coefficients, the theoretical optimum of the introduced objective is close to the optimal transport map (Thm.~\ref{thm1}). Our experiments in Sec.~\ref{subsec:toy_exps} demonstrate how the choice of regularization coefficient affects the trained mapping and allows us to build the general intuition. In Sec.~\ref{subsec:c2w} we compare our method with the diffusion-based baselines (ILVR, SDEdit, EGSDE) and obtain better results given fair restrictions on the transport cost. The results are strictly better than all of the baselines in terms of SSIM and almost strictly superior to all of the baselines in terms of PSNR.

In terms of limitations, we admit that our theory works in the asymptotic regime, while one could derive more precise non-limit bounds. Our experimental results on $Cat \rightarrow Wild$ demonstrate the lowest FID around $6.9$, while the pre-trained diffusion model has $2.01$. Improving the visual quality and testing our method on high dimensions is important for future work. Furthermore, the desired feature of the method would be switching among different reg. coefficients without re-training.

\section*{Acknowledgements}

The work of Denis Rakitin was supported by the grant for research centers in the field of AI provided by the Analytical Center for the Government of the Russian Federation (ACRF) in accordance with the agreement on the provision of subsidies (identifier of the agreement 000000D730321P5Q0002) and the agreement with HSE University No. 70-2021-00139. This research was supported in part through computational resources of HPC facilities at HSE University~\cite{kostenetskiy2021hpc}.

\bibliography{references}
\bibliographystyle{icml2024}

\newpage
\appendix
\onecolumn
\section{Theory}
\label{sec:theory}
In this section, we aim at proving the main theoretical result of the work: solution of the soft-constrained RDMD objective converges to the solution of the hard-constrained Monge problem. Our proof is largely based on the work~\cite{liero2018optimal}. It introduces the family of entropy-transport problems, consisting in optimizing the transport cost with soft constraints based on the divergence between the map's output distribution and the target. There are, however, differences between the problems, that prevent us from reducing the functional in Equation~\ref{eq:rdmd_loss} to the entropy-transport problems. First, authors consider the case of finite non-negative measures, while we stick to the probability distributions. Second, the family of Csisz{\'a}r $f$-divergences~\cite{csiszar1967information}, used in~\cite{liero2018optimal}, seemingly does not contain the integral ensemble of KL divergences, used in Equation~\ref{eq:rdmd_loss}. Finally, we illustrate the proof in a simpler particular setting for the narrative purposes. Nevertheless, the used ideas are very similar.

\subsection{Proof outline}
\label{subsec:proof_outline}
We start by giving a simple outline of the proof. Given a pair of source and target distributions $\psource$ and $\ptarget$, RDMD optimizes the following functional with respect to the generator $G$:
\begin{equation}
    \int\limits_{0}^{\Tmax} \klweight \KL\left(\pgen_t \,\|\, \ptarget_t\right) \rmd t + \lambda\,\E_{\psource(\x)} c\left(\x, G(\x)\right),
\end{equation}
where $\pgen_t$ and $\ptarget_t$ are the generator distribution $\pgen$ and the target distribution $\ptarget$, perturbed by the forward diffusion process up to the time step $t$. Our goal is to prove that the optimal generator of the regularized objective converges to the optimal transport map when $\lambda \rightarrow 0$. With a slight abuse of notation, in this section we will use a different objective
\begin{equation}
\label{eq:rdmd_alpha}
    \loss^\alpha(G) = \alpha\, \int\limits_{0}^{\Tmax} \klweight \KL\left(\pgen_t \,\|\, \ptarget_t\right) \rmd t + \,\E_{\psource(\x)} c\left(\x, G(\x)\right)
\end{equation}
and consider the equivalent limit $\alpha \rightarrow + \infty$. We also define 
\begin{equation}
    \loss^\infty(G) = \begin{dcases}
        \E_{\psource(\x)} c\left(\x, G(\x)\right), \text{ if }\,\pgen = \ptarget; \\
        + \infty, \text{ else}
    \end{dcases}
\end{equation}
to be the objective, corresponding to the unconditional formulation of the Monge problem (Equation~\ref{eq:monge_ot_problem}). In this section, we will denote minimum of this objective (which is, therefore, the optimal transport map) as $G^\infty$~\footnote{Solution to the Monge problem is not always unique, but we will further impose assumptions that will guarantee the uniqueness.}

We first assume that the infimum of the objective $\loss^\alpha$ is reached and define $G^\alpha$ be the optimal generator. We denote by $\{\alpha_n\}_{n = 1}^{+ \infty}$ an arbitrary sequence with $\alpha_n \rightarrow +\infty$. We first make two informal assumptions that need to be proved (and will be in some sence further in the section):
\begin{enumerate}
    \item The sequence $G^{\alpha_n}$ converges (in some sence) to some function $\hat{G}$;
    \item $\loss^\alpha$ is continuous with respect to this convergence, i.e. for every convergent sequence $G_n \rightarrow G$ holds $\loss^\alpha(G_n) \rightarrow \loss^\alpha(G)$.
\end{enumerate}

Given this, we first observe that for each map $G$ the sequence of objectives $\loss^{\alpha_n}(G)$ monotonically converges to the objective $\loss^\infty(G)$. It follows from the fact that the first summand of $\loss^{\alpha_n}$ converges to $+\infty$ if and only if the KL divergence is non-zero, which is equivalent to saying that $\pgen$ and $\ptarget$ differ~\cite{wang2024prolificdreamer}. If instead $\pgen = \ptarget$, the summand zeroes out. This also means that the minimal values of the corresponding objectives form a monotonic sequence:
\begin{equation}
    \loss^{\alpha_n}(G^{\alpha_n}) \leq \loss^{\alpha_{n + 1}}(G^{\alpha_{n + 1}}) \leq \loss^{\infty}(G^\infty).
\end{equation}

Finally, the monotonicity implies that for a fixed $m$
\begin{equation}
    \lim\limits_{n \rightarrow \infty} \loss^{\alpha_n}(G^{\alpha_n}) \geq \lim\limits_{n \rightarrow \infty} \loss^{\alpha_m}(G^{\alpha_n}),
\end{equation}
since the input $G^{\alpha_n}$ is fixed and $\mathcal{L}^{\alpha_n}$ monotonically increases. Using the assumed continuity of the objective, we obtain
\begin{equation}
    \lim\limits_{n \rightarrow \infty} \loss^{\alpha_n}(G^{\alpha_n}) \geq \loss^{\alpha_m}(\hat{G})
\end{equation}
for each $m$. Taking the limit $m \rightarrow \infty$, we obtain
\begin{equation}
    \lim\limits_{n \rightarrow \infty} \loss^{\alpha_n}(G^{\alpha_n}) \geq \loss^\infty(\hat{G}).
\end{equation}
Combining this set of equations, we obtain:
\begin{equation}
    \loss^\infty(G^\infty) \geq \lim\limits_{n \rightarrow \infty} \loss^{\alpha_n}(G^{\alpha_n}) \geq \loss^\infty(\hat{G}) \geq \loss^\infty(G^\infty),
\end{equation}
where the first inequality comes from the monotonicity of the minimal values and the last inequality uses that $G^\infty$ is the minimum of the objective $\loss^\infty$. Hence, that limiting map $\hat{G}$ achieves minimal value of the objective $\loss^\infty$ and is, therefore, the optimal transport map.

At this point, we only need to define and prove some versions of the aforementioned facts:
\begin{enumerate}
    \item Infimum of $\loss^{\alpha}$ is reached;
    \item The sequence of minima $G^{\alpha_n}$ converges;
    \item $\loss^{\alpha}$ is continuous with respect to this convergence.
\end{enumerate}

From now on, we formulate the result in details and stick to the formal proof.

\subsection{Assumptions and theorem statement}

First, we list the assumptions.

\begin{assumption}
\label{A1}
    The distributions $\psource$ and $\ptarget$ have densities with respect to the Lebesgue measure. The distributions are defined on open bounded subsets $\mathcal{X} \subset \Rd$ and $\mathcal{Y} \subset \Rd$, where $\mathcal{Y}$ is convex. The densities are bounded away from zero and infinity on $\mathcal{X}$ and $\mathcal{Y}$, respectively.
\end{assumption}

We admit that boundedness of the support is a very restrictive assumption from the theoretical standpoint, however in our applications (I2I) both source and target distributions are supported on the bounded space of images. We thus can set $\mathcal{X} = \mathcal{Y} = (0, 1)^d$.

\begin{assumption}
\label{A2}
    The cost $c(\x, \y)$ is quadratic $\|\x - \y\|^2$.
\end{assumption}

Here, we stick to proving the theorem only for $L_2$ cost due to difficulties in investigation of Monge map existence and regularity for general transport costs~\cite{de2014monge}.
\begin{assumption}
\label{A3}
    The weighting function $\klweight$ is positive and bounded.
\end{assumption}

\begin{assumption}
\label{A4}
    Standard deviation $\sigma_t$ of the noise, defined by the forward process, is continuous in $t$.
\end{assumption}
 
\setcounter{theorem}{0}
\begin{theorem}
\label{thm1_app}
    Let $\psource$, $\ptarget, c\,, \klweight,$ and $\sigma_t$ satisfy the assumptions \textbf{1-3}. Then, there exists a minimum $G^\alpha$ of the objective $\loss^\alpha$ from the Equation~\ref{eq:rdmd_alpha}. If $\alpha_n \rightarrow \infty$, the sequence $G^{\alpha_n}$ converges in probability (with respect to the source distribution) to the optimal transport map $G^\infty$: 
\begin{equation}
    G^{\alpha_n} \xrightarrow[n \rightarrow \infty]{\psource} G^\infty.
\end{equation}
\end{theorem}

\subsection{Theoretical background}
\label{subsec:weak_conv}
We start by listing all the results necessary for the proof. They are mostly related to the topics of measure theory (weak convergence, in particular) and optimal transport. Most of these classic facts can be found in the books~\cite{bogachev2007measure, dudley2018real}. Otherwise, we make the corresponding citations.
\begin{definition}
    A sequence of probability distributions $p^n(\x)$ converges weakly to the distribution $p(\x)$ if for all continuous bounded test functions $\varphi \in \Cb(\Rd)$ holds
    \begin{equation}
        \E_{p^n(\x)} \varphi(\x) \xrightarrow[n \rightarrow \infty]{} \E_{p(\x)}\varphi(\x).
    \end{equation}
    Notation: $p^n \weak p$.
\end{definition}

\begin{definition}
    A function $f : \Rd \rightarrow \R$ is called lower semi-continuous (lsc), if for all $\x_n \rightarrow \x$ holds
    \begin{equation}
        \liminf\limits_{n \rightarrow \infty} f(\x_n) \geq f(\x).
    \end{equation}
\end{definition}

\begin{theorem}[Portmanteau/Alexandrov]
\label{thm:alexandrov}
    $p^n \weak p$ is equivalent to the following statement: for every lsc function $f$, bounded from below, holds
    \begin{equation}
        \liminf\limits_{n \rightarrow \infty}\E_{p^n(\x)}f(\x) \geq \E_{p(\x)}f(\x).
    \end{equation}
\end{theorem}

\begin{definition}
    A sequence of probability measures $p^n$ is called relatively compact, if for every subsequence $p^{n_k}$ there exists a weakly convergent subsequece $p^{n_{k_j}}$.
\end{definition}

\begin{definition}
    A sequence of probability measures $p^n$ is called tight, if for every $\eps > 0$ there exists a compact set $K_\eps$ such that $p^n(K_\eps) \geq 1 - \eps\,$ for all $n$.
\end{definition}

\begin{theorem}[Prokhorov]
\label{thm:prokhorov}
    A sequence of probability measures $p^n$ is relatively compact if and only if it is tight. In particular, every weakly convergent sequence is tight.
\end{theorem}

\begin{corollary}
\label{thm:prokhorov_bound}
    If there exists a function $\varphi(\x)$ such that its sublevels $\{ \x : \varphi(x) \leq r\}$ are compact and for all $n$
    \[
        \E_{p^n(\x)}\varphi(x) \leq C
    \]
    holds with some constant $C$, then $p^n$ is tight (i.e. at least it has a weakly convergent subsequence).
\end{corollary}

\begin{corollary}
\label{thm:prokhorov_convergence}
    If a sequence $p^n$ is tight and all of its weakly convergent subsequences converge to the same measure $p$, then $p^n \weak p$.
\end{corollary}

\begin{definition}
    The functional $\loss(p)$ is called lower semi-continuous (lsc) with respect to the weak convergence if for all weakly convergent sequences $p^n \weak p$ holds
    \begin{equation}
        \liminf\limits_{n \rightarrow \infty} \loss(p^n) \geq \loss(p).
    \end{equation}
\end{definition}

\begin{theorem}[\citealt{posner1975random}]
\label{thm:kl_lsc}
    The \textup{KL} divergence $\textup{\KL}(p \,\|\, q)$ is lsc (in sense of weak convergence) with respect to each argument, i.e. if $p^n \weak p$ and $q_n \weak q$, then
    \begin{align}
        &\liminf\limits_{n \rightarrow \infty} \textup{\KL}(p^n \,\|\, q) \geq \textup{\KL}(p \,\|\, q)\\
        &\liminf\limits_{n \rightarrow \infty} \textup{\KL}(p \,\|\, q_n) \geq \textup{\KL}(p \,\|\, q).
    \end{align}
\end{theorem}

\begin{theorem}[\citealt{donsker1983asymptotic}]
\label{thm:donsker_varadhan}
    The \textup{KL} divergence can be expressed as
    \begin{equation}
    \label{eq:kl_variational}
        \textup{\KL}(p \| q) = \sup\limits_{g} \left(\E_{p(\x)} g(\x) - \log \E_{q(\x)} e^{g(\x)} \right).
    \end{equation}
\end{theorem}

\begin{definition}
    The expression
    \begin{equation}
        \E_{p(\x)} e^{i \langle s , \x \rangle}
    \end{equation}
    is called the characteristic function (Fourier transform) of the distribution $p(\x)$. 
\end{definition}

\begin{theorem}[Lévy]
\label{thm:continuity}
    Weak convergence of probability measures $p^n \weak p$ is equivalent to the point-wise convergence of characteristic functions, i.e. $\E_{p^n(\x)} e^{i \langle s , \x \rangle} \rightarrow \E_{p(\x)} e^{i \langle s , \x \rangle}$ for all $s$.
\end{theorem}

\begin{definition}
    A sequence of measurable functions $\varphi^n(\x)$ is said to converge in measure (in probability) to the function $\varphi$ with respect to the measure $p(\x)$, if for all $\eps > 0$ holds
    \[
        p\left(\{\x : |\varphi^n(\x) - \varphi(\x)| > \eps \}\right) \rightarrow 0.
    \]
\end{definition}

\begin{theorem}[Lebesgue]
\label{thm:lebesgue}
    Let $\varphi^n, \varphi$ be measurable functions such that $\|\varphi^n(\x)\|, \|\varphi(\x)\| \leq C$ and $\varphi^n(\x) \rightarrow \varphi(\x)$ pointwise. Then $\E_{p(\x)}\varphi^n(\x) \rightarrow \E_{p(\x)}\varphi(\x)$.
\end{theorem}

\begin{lemma}[Fatou]
\label{thm:fatou}
    For any sequence of measurable functions $\varphi^n$ the function $\liminf_n \varphi^n$ is measurable and
    \begin{equation}
        \int\limits_{a}^{b}\liminf\limits_{n \rightarrow \infty} \varphi^n(\x) \rmd \x \leq \liminf\limits_{n \rightarrow \infty} \int\limits_{a}^{b} \varphi^n(\x) \rmd \x.
    \end{equation}
\end{lemma}

\begin{theorem}[\citealt{brenier1991polar}]
\label{thm:brenier}
    Given the Assumption~\ref{A1}, there exists a unique optimal transport map that solves the Monge problem~\ref{eq:monge_ot_problem} for the quadratic cost.
\end{theorem}
\begin{proof}
This result can be found e.g. in~\citep[Theorem~3.1]{de2014monge}.
\end{proof}
\begin{theorem}
\label{thm:ot_continuous}
    Given the Assumption~\ref{A1}, the unique OT Monge map is continuous.
\end{theorem}
\begin{proof}
This is a simplified version of~\citep[Theorem~3.3]{de2014monge}.
\end{proof}

\subsection{Lower semi-continuity of the loss}
\label{subsec:loss_lsc}
Having defined all the needed terms and results, we start the proof by re-defining the objective in Equation~\ref{eq:rdmd_alpha} with respect to the joint distribution $\pi$ input and output of the generator instead of the generator $G$ itself. Analogous to the Kantorovitch formulation of the optimal transport problem~\cite{kantorovitch1958translocation}, for each measure $\pi$ on $\Rd \times \Rd$ (which is also called a \emph{transport plan} or just plan) we define the corresponding fuctional as
\begin{equation}
    \loss^\alpha(\pi) = \alpha\, \int\limits_{0}^{\Tmax} \klweight \KL\left(\pi_{\y, t} \,\|\, \ptarget_t\right) \rmd t + \,\E_{\pi(\x, \y)} c\left(\x, \y\right),
\end{equation}
where $\pi_{\x}$ and $\pi_{\y}$ are the corresponding projections (marginal distributions) of $\pi$ and $\pi_{\y, t}$ is the perturbed $\y$-marginal distribution of $\pi$.  Note that for $\pi$, corresponding to the joint distribution of $(\x, G(\x))$, $\loss^\alpha(\pi)$ coincides with $\loss^\alpha(G)$, defined in Equation~\ref{eq:rdmd_alpha}. Thus, we aim to optimize $\loss^\alpha(\pi)$ with respect to such plans $\pi$, that their $\x$ marginal is equal to $\psource$ and $\pi(\y = G(\x)) = 1$ for some $G$.
\begin{definition}
\label{def:generator_based}
    We will call a measure $\pi$ generator-based if its $\x$-marginal is equal to $\psource$ and $\pi(\y = G(\x))$ for some function $G$.
\end{definition}

For the sake of clearity, we note that the distributions $\pi^{\y}_t$ and $\ptarget_t$ can be represented as $\pi^{\y} * q_t$ and $\ptarget * q_t$, where $*$ is the convolution operation and $q_t = \N(0, \sigma_t^2 I)$. We thus rewrite the functional as
\begin{equation}
    \loss^\alpha(\pi) = \alpha\, \int\limits_{0}^{\Tmax} \klweight \KL\left(\pi_{\y} * q_t \,\|\, \ptarget * q_t \right) \rmd t + \,\E_{\pi(\x, \y)} c\left(\x, \y\right),
\end{equation}
Previously, we wanted to establish continuity of the objective. This may not be the case in general. Instead, we prove the following
\begin{lemma}
    $\loss^\alpha(\pi)$ is lsc with respect to the weak convergence, i.e. for all weakly convergent sequences $\pi^n \weak \pi$ holds 
    \begin{equation}
        \liminf\limits_{n \rightarrow \infty} \loss^\alpha(\pi^n) \geq \loss^\alpha(\pi).
    \end{equation}
\end{lemma}
This result is a direct consequence of the Theorem~\ref{thm:kl_lsc} about lower semi-continuity of the KL divergence.
\begin{proof}
     We start by proving that the projection and the convolution operation preserve weak convergence. For the first, we need to prove that for any test function $g \in \Cb(\Rd)$ holds \begin{equation}
         \E_{\pi^{n}_{\y}(\y)} g(\y) \rightarrow \E_{\pi_{\y}(\y)} g(\y)
     \end{equation} given $\pi^n \weak \pi$. For this, we note that the function $\varphi(\x, \y) = g(\y)$ is also bounded and continuous and, thus
     \begin{equation}
         \E_{\pi^{n}_{\y}(\y)} g(\y) = \E_{\pi^n(\x, \y)}\varphi(\x, \y) \rightarrow \E_{\pi(\x, \y)}\varphi(\x, \y) = \E_{\pi_{\y}(\y)} g(\y).
     \end{equation}
     Regarding the convolution, recall that $\pi^n_{\y} * q_t$ is the distribution of the sum of independent variables with corresponding distributions. Its characteristic function is equal to
     \begin{equation}
         \E_{\pi^n_{\y} * q_t (\y_t)} e^{i \langle s, \y_t \rangle} = \E_{\pi^n_{\y}(\y)q_t(\eps_t)}e^{i \langle s, \y + \eps_t \rangle} = \E_{\pi^n_{\y}(\y)}e^{i\langle s, \y \rangle}\E_{q_t(\eps_t)}e^{i\langle s, \eps_t \rangle}.
     \end{equation}
    Applying the Lévy's continuity theorem to $\pi^n_{\y} \weak \pi_{\y}$, we take the limit and obtain
    \begin{equation} 
        \E_{\pi_{\y}(\y)}e^{i\langle s, \y \rangle}\E_{q_t(\eps_t)}e^{i\langle s, \eps_t \rangle} = \E_{\pi_{\y}(\y)q_t(\eps_t)}e^{i \langle s, \y + \eps_t \rangle} = \E_{\pi_{\y} * q_t (\y_t)} e^{i \langle s, \y_t \rangle},
    \end{equation}
    which implies 
    \begin{equation}
        \E_{\pi^n_{\y} * q_t (\y_t)} e^{i \langle s, \y_t \rangle} \rightarrow \E_{\pi_{\y} * q_t (\y_t)} e^{i \langle s, \y_t \rangle}.
    \end{equation}
    We apply the continuity theorem for the convolutions and obtain $\pi^n_{\y} * q_t \weak \pi_{\y} * q_t$.

    With this observation, we prove that the first term of $\loss^\alpha(\pi)$ is lsc. First, we apply Lemma~\ref{thm:fatou} (Fatou) and move the limit inside the integral
    \begin{equation}
    \label{eq:fatou_inside_kl}
        \liminf\limits_{n \rightarrow \infty}\int \limits_{0}^{\Tmax} \klweight \KL\left(\pi^n_{\y} * q_t \,\|\, \ptarget * q_t \right) \rmd t \geq \int \limits_{0}^{\Tmax} \liminf\limits_{n \rightarrow \infty}\klweight \KL\left(\pi^n_{\y} * q_t \,\|\, \ptarget * q_t \right) \rmd t.
    \end{equation}
    Using the lower semi-continuity of the KL divergence (Theorem~\ref{thm:kl_lsc}), we obtain
    \begin{equation}
    \label{eq:lsc_inside_kl}
        \int \limits_{0}^{\Tmax} \liminf\limits_{n \rightarrow \infty}\klweight \KL\left(\pi^n_{\y} * q_t \,\|\, \ptarget * q_t \right) \rmd t \geq \int\limits_{0}^{\Tmax} \klweight \KL \left(\pi_{\y} * q_t \,\|\, \ptarget * q_t\right) \rmd t.
    \end{equation}
    Finally, the Assumption~\ref{A2} on the continuity of $c(\cdot, \cdot)$ implies its lower semi-coninuity. Theorem~\ref{thm:alexandrov} (Portmanteau) states that 
    \begin{equation}
    \label{eq:alexandrov_for_ot}
        \liminf\limits_{n \rightarrow \infty} \E_{\pi^n(\x, \y)} c(\x, \y) \geq \E_{\pi(\x, \y)}c(\x, \y).
    \end{equation}
    Combining inequalities from Equation~\ref{eq:fatou_inside_kl}, Equation~\ref{eq:lsc_inside_kl} and Equation~\ref{eq:alexandrov_for_ot}, we obtain
    \begin{equation}
        \liminf\limits_{n \rightarrow \infty} \loss^\alpha(\pi^n) \geq \loss^\alpha(\pi).
    \end{equation}
\end{proof}
\subsection{Existence of the minimizer}
\label{subsec:inf}
Now we aim to prove that the objective $\loss^\alpha(\pi)$ has a minimum over generator-based plans. 
First, we need the following technical lemma about sublevels of the KL part of the functional.
\begin{lemma}
\label{lemma:tightness}
    Let $\{\pi^{n}\}_{n = 1}^{\infty}$ be a sequence of generator-based plans that satisfy 
    \begin{equation}
        \int\limits_{0}^{\Tmax} \klweight \textup{\KL}\left(\pi^n_{\y, t} \,\|\, \ptarget_t \right) \rmd t \leq C
    \end{equation}
    for some constant $C$. Then, the sequence $\{\pi^n\}_{n = 1}^{\infty}$ is tight.
\end{lemma}
\begin{proof}
    We take arbitrary $\pi$ from the sequence and apply the Donsker-Varadhan representation (Theorem~\ref{thm:donsker_varadhan}) of the KL divergence. We take the test function $g(\x) = \|x\|^2 / (2\sigma_{\Tmax}^2)$ and obtain
    \begin{equation}
        \int\limits_{0}^{\Tmax} \klweight \KL\left(\pi_{\y, t} \,\|\, \ptarget_t\right) \rmd t \geq \int\limits_{0}^{\Tmax} \klweight \left(\E_{\pi_{\y, t}(\y_t)} \frac{1}{2\sigma_{\Tmax}^2} \|\y_t\|^2 - \log \E_{\ptarget_t(\y_t)}e^{ \|\y_t\|^2 / (2\sigma_{T}^2)} \right) \rmd t.
    \end{equation}
    The choice of $g(\x)$ is not very specific, i.e. every function that will produce finite expectations and integrals is suitable. In the right-hand side, we rewrite the expectations with repect to the original variable and noise:
    \begin{equation}
    \int\limits_{0}^{\Tmax} \klweight \left(\E_{\pi_{\y}(\y)\N(\eps | 0, I)}\frac{1}{2\sigma_{\Tmax}^2} \|\y + \sigma_t \eps \|^2 - \log \E_{\ptarget(\y)\N(\eps | 0, I)}e^{\|\y + \sigma_t \eps \|^2 / (2 \sigma_{\Tmax}^2)}\right) \rmd t.
    \end{equation}
    We rewrite $\|\y + \sigma_t \eps \|^2$ as $\|\y\|^2 + 2 \sigma_t \langle \y, \sigma_t \eps \rangle + \sigma_t^2 \|\eps\|^2$ and note that expectation of the second term is zero. The first term is then equal to
    \begin{equation}
        \frac{1}{2\sigma_{\Tmax}^2} \int\limits_{0}^{\Tmax} \klweight \rmd t \cdot \E_{\pi_{\y}(\y)}\|\y\|^2 + \frac{1}{2\sigma_{\Tmax}^2} \int\limits_{0}^{\Tmax} \klweight \sigma_{t}^2 \rmd t \cdot \E_{\N(\eps | 0, I)} \|\eps \|^2.
    \end{equation}
    Boundedness of $\omega_t$ (Assumption~\ref{A3}) implies that the first integral is finite and, say, equal to $C_1$. The second integral contains a product of bounded $\omega_t$ and continuous $\sigma_t^2$ (Assumtion~\ref{A4}), which is also integrable. We then denote the second summand by $C_2$ and rewrite the first summand as
    \begin{equation}
        C_1 \E_{\pi_{\y}(\y)} \| \y \|^2 + C_2.
    \end{equation}
    As for the second summand, we see that the expectation    \begin{equation}
        \E_{\ptarget(\y)\N(\eps | 0, I)}e^{\|\y + \sigma_t \eps\|^2 / (2 \sigma_{\Tmax}^2)}
    \end{equation}
    with respect to $\eps$ will be finite, because $\sigma_t^2 / (2\sigma_{\Tmax}^2)$ is always less than $1/2$, which will make the exponent have negative degree. Moreover, simple calculations show that this function will be continuous with respect to $\sigma_t$ and have only quadratic terms with respect to $\y$ inside the exponent, i.e. have the form
    \begin{equation}
        e^{a(\sigma_t) \|\y - b(\sigma_t)\|^2 + c(\sigma_t)}
    \end{equation}
    with continuous $a, b, c$. We now want to prove that the expectation 
    \begin{equation}
        \E_{\ptarget(\y)}e^{\alpha(\sigma_t) \|\y - \beta(\sigma_t)\|^2 + \gamma(\sigma_t)}
    \end{equation}
    will also be continuous in $t$. First, due to the boundedness of $\y$, this expectation is finite. Second, for $t_n \rightarrow t$:
    \begin{align}
        \lim\limits_{n \rightarrow \infty}&\E_{\ptarget(\y)}e^{a(\sigma_{t_n}) \|\y - b(\sigma_{t_n})\|^2 + c(\sigma_{t_n})} = \\
        = \:&\E_{\ptarget(\y)} \lim\limits_{n \rightarrow \infty} e^{a(\sigma_{t_n}) \|\y - b(\sigma_{t_n})\|^2 + c(\sigma_{t_n})} = \\
        = \:&\E_{\ptarget(\y)} e^{a(\sigma_{t}) \|\y - b(\sigma_{t})\|^2 + c(\sigma_{t})}
    \end{align}
    due to the Theorem~\ref{thm:lebesgue} (Lebesgue's dominated convergence). It is applicable, since $\y$ is bounded and all the functions are continuous, thus bounded in $[0, T]$.

    We thus obtain that the second integral contains bounded $\omega_t$ multiplied by the logarithm of continuous function, which is always $\geq 1$ (positive exponent). This means that the whole integral is finite. Denoting it by $C_3$, we obtain
    \begin{equation}
        C_1 \E_{\pi_{\y}(\y)} \|\y\|^2 + C_2 - C_3\leq \int\limits_{0}^{\Tmax} \klweight \KL\left(\pi_{\y, t} \,\|\, \ptarget_t\right) \rmd t.
    \end{equation}
    Combined with the condition of the lemma, we obtain
    \begin{equation}
        C_1 \E_{\pi_{\y}(\y)} \|\y\|^2 + C_2 - C_3 \leq \int\limits_{0}^{\Tmax} \klweight \KL\left(\pi_{\y, t} \,\|\, \ptarget_t\right) \rmd t  \leq C,
    \end{equation}
    which implies
    \begin{equation}
        \E_{\pi_{\y}(\y)} \|\y\|^2 \leq \frac{C + C_3 - C_2}{C_1} := C_4.
    \end{equation}
    We thus obtained a uniform bound on some statistic with respect to all measures from $\{\pi^n\}$. The function $\|\y\|^2$ has compact sublevel sets $\{ \|\y\|^2 \leq r\}$. Lemma~\ref{thm:prokhorov_bound} then states that the sequence $\pi^{n}_{\y}$ is tight, i.e. for all $\eps > 0$ there is a compact set $K_\eps$ with $\pi^{n}_{\y}(\y \in K_\eps) \geq 1 - \eps$. 

    Finally, marginal $\x$ distribution of each of the $\pi^{n}$ is $\psource$, which is bounded (Assumption~\ref{A1}), i.e. there is a compact $K$ that $\pi^n(\x \in K) = 1$. Combined with the previous observation, we obtain
    \begin{equation}
        \pi^n(\x \in K, \y \in K_\eps) \geq 1 - \eps
    \end{equation}
    for all $n$. The cartesian product $K \times K_\eps$ is also compact. Theorem~\ref{thm:prokhorov} (Prokhorov) then implies that the sequence $\pi^n$ is tight.
\end{proof}

Now we are ready to prove the following
\begin{lemma}
    Infimum of the loss $\loss^\alpha(\pi)$ over all generator-based transport plans $\pi$ (with $\pi_{\x} = \psource$ and $\pi(\y = G(\x))$ for some $G$) is attained on some plan $\hat{\pi}$.
\end{lemma}
\begin{proof}
    We start by observing that there is at least one feasible $\pi$ with the aforementioned properties. For this purpose one can take the optimal transport map $G^\infty$ between $\psource$ and $\ptarget$, which is unique by Theorem~\ref{thm:brenier} under Assumptions~\ref{A1},~\ref{A2}.

    Let $\pi^n$ be a sequence of feasible generator-based measures that $\loss^\alpha(\pi^n)$ converges to the corresponding infimum $\loss^\alpha_{\inf}$ (it exists by the definition of the infimum). Without loss of generality, we can assume that $\loss^\alpha(\pi^n) \leq \loss^\alpha_{\inf} + 1$ for all $n$ (if not, one can drop large enough sequence prefix). This implies that for all $n$ holds

    \begin{equation}\alpha\, \int\limits_{0}^{\Tmax} \klweight \KL\left(\pi_{\y, t} \,\|\, \ptarget_t\right) \rmd t \leq \loss^\alpha_{\inf} + 1.
    \end{equation}

    Lemma~\ref{lemma:tightness} implies that the sequence $\pi^n$ is tight. Prokhorov theorem then states that $\pi^n$ has a weakly convergent subsequence $\pi^{n_k} \weak \hat{\pi}$. Lower semi-continuity of the loss $\loss^\alpha$ implies that
    \begin{equation}
        \liminf\limits_{k \rightarrow \infty} \loss^\alpha(\pi^{n_k}) \geq \loss^\alpha(\hat{\pi}) \geq \loss^\alpha_{\inf}.
    \end{equation}
    At the same time, $\loss^\alpha(\pi^{n_k})$ is assumed to converge to $\loss^\alpha_{\inf}$, which means that $\hat{\pi}$ is indeed the minimum.
\end{proof}

\subsection{Finish of the proof}
\label{subsec:proof}
\begin{proof}[Theorem~\ref{thm1_app} proof]
Finally, we combine the previous technical observations with the proof sketch from the Section~\ref{subsec:proof_outline}. Let $\alpha_n \rightarrow \infty$ be a sequence of coefficients, $G^{\alpha_n}$ be the optimal generators with respect to $\loss^{\alpha_n}$ and $\pi^{\alpha_n}$ the joint distributions of $(\x, G^{\alpha_n}(\x))$. Additionally, we define $\pi^\infty$ to be the optimal transport plan, corresponding to $(\x, G^\infty(\x))$, where $G^\infty(\x)$ is the optimal transport map. First, due to the monotonicity of $\loss^\alpha$ with respect to $\alpha$, we have
\begin{equation}
    \loss^{\alpha_n}(\pi^{\alpha_n}) \leq \loss^{\alpha_{n + 1}}(\pi^{\alpha_{n + 1}}) \leq \loss^{\infty}(\pi^\infty).
\end{equation}
This implies that for all $n$ holds
\begin{align}
    \alpha_n &\int\limits_{0}^{\Tmax} \klweight \KL\left(\pi^{\alpha_n}_{\y, t} \,\|\, \ptarget_t\right) \rmd t \leq \loss^\infty(\pi^\infty) \Rightarrow \\
    \Rightarrow & \int\limits_{0}^{\Tmax} \klweight \KL\left(\pi^{\alpha_n}_{\y, t} \,\|\, \ptarget_t\right) \rmd t \leq \frac{\loss^\infty(\pi^\infty)}{\alpha_n} \leq \frac{\loss^\infty(\pi^\infty)}{\min\limits_{n} \alpha_n},
\end{align}
which is finite, since $\alpha_n \rightarrow +\infty$. One more time, we apply Lemma~\ref{lemma:tightness} and conclude that the sequence $\pi^{\alpha_n}$ is tight.

Let $\pi^{\alpha_{n_k}}$ be its weakly convergent subsequence: $\pi^{\alpha_{n_k}} \weak \hat{\pi}$. Analogously to the Section~\ref{subsec:proof_outline}, we observe that
\begin{equation}
    \liminf\limits_{k \rightarrow \infty} \loss^{\alpha_{n_k}} (\pi^{\alpha_{n_k}}) \geq \liminf\limits_{k \rightarrow \infty}\loss^{\alpha_{n_m}}(\pi^{\alpha_{n_k}}) \geq \loss^{\alpha_{n_m}}(\hat{\pi})
\end{equation}
for any fixed $m$. The first inequality is due to the monotonicity of $\loss^\alpha$ with respect to $\alpha$ and second is the implication of lower semi-continuity of the loss $\loss^\alpha$ with respect to weak convergence. Taking the limit $m \rightarrow \infty$, we obtain
\begin{equation}
     \liminf\limits_{k \rightarrow \infty} \loss^{\alpha_{n_k}} (\pi^{\alpha_{n_k}}) \geq \loss^{\infty}(\hat{\pi}).
\end{equation}

Combining all these observations, we obtain the following sequence of inequalities
\begin{equation}
    \loss^{\infty}(\pi^\infty) \geq \liminf\limits_{k \rightarrow \infty} \loss^{\alpha_{n_k}} (\pi^{\alpha_{n_k}}) \geq \loss^\infty(\hat{\pi}) \geq \loss^\infty(\pi^\infty),
\end{equation}
which implies that the limiting measure $\hat{\pi}$ reaches the minimum of the objective over generator-based plans. By the uniqueness of the optimal transport map $G^\infty$ under the Assumptions~\ref{A1},~\ref{A2},~\ref{A3}, we conclude that all the convergent subsequences $\pi^{\alpha_{n_k}}$ converge to the optimal measure $\pi^\infty$. Using Corollary~\ref{thm:prokhorov_convergence} of the Prokhorov theorem, we deduce that $\pi^{\alpha_n} \weak \pi^\infty$.

Finally, we want to replace the weak convergence of $\pi^{\alpha_n}$ to $\pi^\infty$ with the convergence in probability of the generators, i.e. show
\begin{equation}
    G^{\alpha_n} \xrightarrow[]{\psource} G^\infty.
\end{equation}
To this end, we represent the corresponding probability as the expectation of the indicator and upper bound it with a continuous function:
\begin{align}
    \psource\left(\|G^{\alpha_n}(\x) - G^{\infty}(\x)\| > \eps\right) &= \E_{\psource(\x)}I\{\|G^{\alpha_n}(\x) - G^{\infty}(\x)\| > \eps\} \\
    &\leq \E_{\psource(\x)} d\left(G^{\alpha_n}(\x), G^{\infty}(\x)\right),
\end{align}
where $d$ is a continuous indicator approximation, defined as
\begin{equation}
    d(\bu, \bv) = \begin{cases}
        \frac{\|\bu - \bv \|}{\eps}, &\text{ if }\, 0 \leq \|\bu - \bv \| < \eps;\\
        1, &\text{ if }\, \|\bu - \bv\| \geq \eps.
    \end{cases}
\end{equation}
We define the test function
\begin{equation}
    \varphi(\x, \y) = d\left(\y, G^\infty(\x)\right)
\end{equation}
and rewrite the upper bound as
\begin{equation}
    \E_{\psource(\x)} d\left(G^{\alpha_n}(\x), G^{\infty}(\x)\right) = \E_{\psource(\x)} \varphi(\x, G^{\alpha_n}(\x)) = \E_{\pi^{\alpha_n}(\x, \y)} \varphi(\x, \y).
\end{equation}

Due to Assumptions~\ref{A1},~\ref{A2} and Theorem~\ref{thm:continuity} the optimal transport map $G^\infty$ is continuous, which implies that this test function is bounded and continuous. Given the weak convergence of $\pi^{\alpha_n}$, we have
\begin{align}
    \E_{\pi^{\alpha_n}(\x, \y)} \varphi(\x, \y) &\rightarrow \E_{\pi^\infty(\x, \y)}\varphi(\x, \y) = \E_{\psource(\x)}\varphi(\x, G^\infty(\x)) =\\
    &= \E_{\psource(\x)} d(G^\infty(\x), G^\infty(\x)) = 0,
\end{align}
which implies the desired
\begin{equation}
    \psource\left(\|G^{\alpha_n}(\x) - G^{\infty}(\x)\| > \eps\right) \rightarrow 0.
\end{equation}
\end{proof}
\section{Ablation of the initialization parameter}
\label{sec:sigma_ablation}
In this section, we further explore the design space of our method by investigating the effect of the fixed generator input noise parameter $\sigma$ on the resulting quality.  To this end, we take the colored version of the MNIST~\cite{lecun1998mnist} data set and perform translation between the digits "2" and "3" initializing from various $\sigma$. We use a small UNet architecture from~\cite{gushchin2024entropic}.

The parameter $\sigma$ is residual from the pre-trained diffusion architecture and therefore fixed throughout training and evaluation. However, the target denoiser network tries to convert the expected noisy input into the corresponding sample from the output distribution. Consequently, one may expect that at a suitable noise level, the generator may change the input's details to make them look appropriate for the target while preserving the original structural properties.

We demonstrate this effect on various noise levels in Figure~\ref{fig:mnist_sigma}. Here we observe that the small sigmas lead to the mapping close to the identity, whereas the large sigmas lead to almost constant blurry images, corresponding to the average "3" of the data set. However, there is a segment $[1.0, 10.0]$ of levels that gives a moderate-quality mapping in terms of both faithfulness and realism, which makes it a suitable initial point. Note that the FID-L2 plot is not monotone at high L2 values due to the overall poor quality of the generator, i.e. it outputs bad-quality pictures slightly related to the source.
\begin{figure}
\centering
\includegraphics[width=0.9\linewidth]{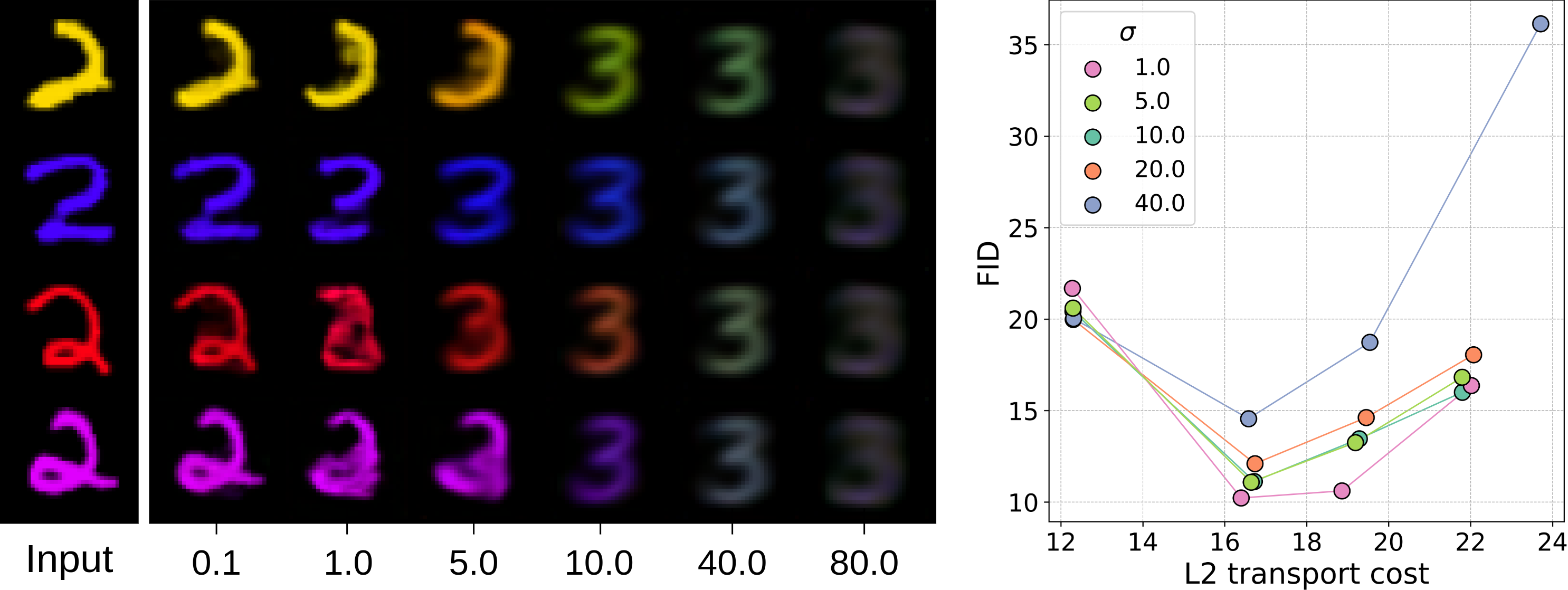}
\caption{Left: visualization of the generator initialization at various $\sigma~\in~[0.1, 80.0]$, where $\sigma$ is the noise level parameter residual from the pre-trained diffusion architecture. Right: comparison of different $\sigma$ in terms of the quality-faithfulness trade-off. The metrics are obtained by initializing the generator at the corresponding $\sigma$ level and training it with the RDMD procedure. Here, $\lambda \in \{0, 1.0, 2.0, 4.0\}$. Higher $\lambda$ corresponds to the lower transport cost values.}
\label{fig:mnist_sigma}
\end{figure}
We further investigate optimal $\sigma$ choice by going through a 2D grid of the hyperparameters $(\sigma, \lambda)$ and aim to see if it is possible to choose the uniform best noise level. In Figure~\ref{fig:mnist_sigma} we report the faithfulness-quality trade-off concerning various $\sigma$. We observe that there is almost monotone dependence on $\sigma$ on the segment $[1.0, 40.0]$: here the $\sigma = 1.0$ gives almost uniformly best results in terms of both metrics. Similar results are obtained by the values $5.0, 10.0$ which have fair quality visual results at initialization. Therefore, we conclude that it is best to choose the least parameter $\sigma$ among the parameters with appropriate visuals at the initial point.
\section{Experimental Details}
\label{sec:exp_details}
\subsection{2D experiments}
\label{subsec:2d_details}
\paragraph{Architecture.} The architecture used for training diffusion model and generator ~\cite{de2021diffusion} consists of input-encoding MLP block, time-encoding MLP block, and decoding MLP block. Input encoding MLP block consists of 4 hidden layers with dimensions $\left[16, 32, 32, 32\right]$ interspersed by LeakyReLU activations. Time encoding MLP consists of a positional encoding layer \cite{NIPS2017_3f5ee243} and then follows the same MLP block structure as the input encoder. The decoding MLP block consists of 5 hidden layers with dimensions $\left[128, 256, 128, 64, 2\right]$ and operates on concatenated time embedding and input embedding each obtained from their respective encoder.
The model contains $88k$ parameters.

\paragraph{Training Diffusion Model.} Diffusion model was trained for 100k iterations with batch size 1024 with Adam optimizer ~\cite{Kingma2014AdamAM} with learning rate $10^{-4}$.

\paragraph{Training RDMD.} Fake denoising network was trained with Adam optimizer with learning rate $10^{-4}$. The generator model was trained with a different Adam optimizer with a learning rate equal to $2 \cdot 10^{-5}$. We trained RDMD for 100k iterations with batch size 1024.

\paragraph{Computational resources.} All the experiments were run on CPU. Running 100k iterations with the batch size 1024 took approximately 1 hour.


\subsection{Colored MNIST}
\label{subsec:mnist_details}

\paragraph{Architecture.} We used the architecture from~\cite{gushchin2024entropic}, which utilizes convolutional UNet with conditional instance normalization on time embeddings used after each upscaling block of the decoder\footnote{\url{https://github.com/ngushchin/EntropicNeuralOptimalTransport/blob/06efb6ba8b43865a30b0b626384fa64da39bc385/src/cunet.py}}. Model produces time embeddings via positional encoding. The model size was approximately $9.9M$ parameters.

\paragraph{Training Diffusion Model.} The diffusion model was trained for 24500 iterations with batch size 8192. We used the Adam optimizer with a learning rate equal to $4 \cdot 10^{-3}$. The model was trained in FP32. It obtained FID equal to 2.09.

\paragraph{Training RDMD.} Fake denoising network was trained with Adam optimizer with a learning rate equal to $2 \cdot 10^{-3}$. The generator model was trained with Adam optimizer with learning rate $5 \cdot 10^{-5}$. RDMD was trained for 7300 iterations with batch size 4096.

\paragraph{Computational resources.} All the experiments were run on 2x NVIDIA GeForce RTX 4090 GPUs. Training Diffusion model for 24500 iterations with batch size 8192 took approximately 6 hours. Training RDMD for 7300 iterations with batch size 4096 took approximately 3 hours.
\subsection{Cat2Wild}
\label{subsec:c2w_details}
\paragraph{Architecture.} We used the SongUNet architecture from EDM repository\footnote{\url{https://github.com/NVlabs/edm/blob/008a4e5316c8e3bfe61a62f874bddba254295afb/training/networks.py}}, which corresponds to DDPM++ and NCSN++ networks from the work~\cite{song2020score}. The model contains approximately $55M$ parameters.

\paragraph{Training Diffusion Model.} The diffusion model was trained for 80k iterations. We set the batch size to 512 and chose the best checkpoint according to FID. We used the Adam optimizer with a learning rate equal to $2 \cdot 10^{-4}$. We used a dropout rate equal to $0.25$ during training and the augmentation pipeline from ~\cite{karras2022elucidating} with probability $0.15$. The model was trained in FP32. It obtained FID equal to 2.01.

\paragraph{Training RDMD.} In all runs we initialized the generator from the target diffusion model with the fixed $\sigma=1.0$. We've run 4 models, corresponding to the regularization coefficients $\{0.02, 0.05, 0.1, 0.2\}$. All models were trained with Adam optimizer with generator learning rate $5 \cdot 10^{-5}$ and fake diffusion learning rate $3 \cdot 10^{-4}$. We trained all models for 25000 iterations with batch size 512. Training took approximately 35 hours on $4\times$ NVidia Tesla A100 80GB.

\paragraph{ILVR hyperparameters.}
The only hyperparameter of ILVR is the downsampling factor $N$ for the low-pass filter, which determines whether guidance would be conducted on coarser or finer information. $n_{\text{steps}}$ denotes number of sampling steps. All metrics in Figure~\ref{fig:c2w_metrics} for ILVR were obtained on the following hyperparameter grid:
$N = [2, 4, 8, 16, 32]$, $n_{\text{steps}} = [18, 32, 50]$. We excluded runs that have the same statistical significance and achieve FID higher than $20.0$. The images in Figure~\ref{fig:c2w_comparison} were obtained with hyperparameters $N = 16$ and $n_{\text{steps}} = 18$.
\paragraph{SDEdit hyperparameters.}
The only hyperparameter of SDEdit is the noise level $\sigma$, which acts as a starting point for sampling. The higher the noise level, the closer is the sampling procedure to the unconditional generation. The smaller the noise values, the more features are carried over to the target domain at the expense of generation quality.  $n_{\text{steps}}$ denotes number of sampling steps. All metrics in Figure~\ref{fig:c2w_metrics} for SDEdit were obtained on the following hyperparameter grid:
$\sigma = [4, 5, 10, 15, 20, 30, 40]$, $n_{\text{steps}} = [18, 32, 50]$. We exclude runs that have the same statistical significance and achieve FID higher than $20.0$. The images in Figure~\ref{fig:c2w_comparison} were obtained with hyperparameters $\sigma = 10$ and $n_{\text{steps}} = 50$.
\paragraph{EGSDE hyperparameters.}
EGSDE sampling hyperparameters include the initial noise level $\sigma$ at which the source image is perturbed, and the downsampling factor $N$ for the low-pass filter. $n_{\text{steps}}$ denotes number of sampling steps. All metrics in Figure~\ref{fig:c2w_metrics} for EGSDE were obtained on the following hyperparameter grid:
$\sigma = [5, 10, 15, 20, 40]$, $N = [8, 16, 32]$, $n_{\text{steps}} = [18, 32]$.
We exclude runs that have the same statistical significance and achieve FID higher than $20.0$. The images in Figure~\ref{fig:c2w_comparison} were obtained with hyperparameters $\sigma=10, N=32, n_{\text{steps}} = 50$.


\end{document}